\documentclass[manuscript,acmsmall,nonacm]{acmart}
\AtBeginDocument{%
  }

\acmJournal{IMWUT}
\setcopyright{acmlicensed}
\copyrightyear{2026}
\acmYear{2026}
\acmDOI{XXXXXXX.XXXXXXX}

\usepackage{amsmath,amsthm}
\usepackage{booktabs}
\usepackage{graphicx}
\usepackage{xcolor}
\usepackage{hyperref}
\usepackage{cleveref}
\usepackage{multirow}
\usepackage{enumitem}
\usepackage{balance}
\newtheorem{proposition}{Proposition}
\newtheorem{lemma}{Lemma}
\newtheorem{theorem}{Theorem}
\newtheorem{corollary}{Corollary}
\newtheorem{definition}{Definition}
\begin{document}

\title{Authority Inversion in LLM-Mediated Ubiquitous Systems: When Models Trust Users Over Sensors}
\author{Long Zhang}
\affiliation{%
  \institution{School of Computer Science and Engineering, South China University of Technology}
  \city{Guangzhou}
  \country{China}
}
\author{Zi-bo Qin}
\affiliation{%
  \institution{School of Computer Science and Engineering, South China University of Technology}
  \city{Guangzhou}
  \country{China}
}
\author{Wei-neng Chen}
\authornote{Corresponding author.}
\affiliation{%
  \institution{School of Computer Science and Engineering, South China University of Technology}
  \city{Guangzhou}
  \country{China}
}
\email{cschenwn@scut.edu.cn}

\renewcommand{\shortauthors}{Zhang et al.}

\setcopyright{none}                
\settopmatter{printacmref=false}  
\renewcommand\footnotetextcopyrightpermission[1]{} 

\begin{abstract}
Large language models (LLMs) increasingly fuse heterogeneous inputs in ubiquitous systems. Yet, how LLMs implicitly allocate authority when sensor measurements and user claims conflict remains unexamined, raising critical reliability concerns for deployments where physical sensing must retain priority. Unlike explicit traditional fusion, LLMs bury authority allocation within learned representations. We discover this allocation is severely format-dependent: numerical sensor data fails to integrate into answer-relevant model directions, allowing natural-language claims to dominate the final decision, a phenomenon we term \textbf{Authority Inversion}.To diagnose and mitigate this, we develop a geometric framework of context integration, introduce two computable audit metrics, specifically the Context Integration Ratio (CIR) and Authority Alignment Index (AAI), and propose Geometric Authority Calibration (GAC), an inference-time layer-level intervention to suppress misplaced user authority. Evaluating four models (4B to 35B parameters, three architectures) across four datasets totaling 576 conflict instances reveals extreme inversion: on numerical tasks, models exhibit near-zero sensor trust (AAI = -0.805, Cohen's d = -2.14), unaffected by model capacity. Validating our geometric framework, theory-guided causal injection flips 80.2\% of incorrect decisions (vs. <0.4\% for random controls). Practically, GAC improves HAR accuracy from 0---1.6\% to 21.9---27.5\%, outperforming prompting baselines. Ultimately, authority allocation in LLM-mediated systems must be explicitly audited and application-specifically configured rather than left implicit.
\end{abstract}




\maketitle

\section{Introduction}
Consider a common ubiquitous computing scenario: a smartwatch continuously records accelerometer and gyroscope signals whose motion patterns strongly indicate that the user is walking, yet the user reports, ``I'm sitting at my desk working.'' In an LLM-mediated context inference pipeline, these two sources of evidence may be presented together to support downstream reasoning. When they disagree, which source does the model effectively trust? Our experiments reveal a striking pattern: across multiple model families and scales, the model often follows the user's natural-language claim and discounts the sensor evidence, even when the sensor measurement is the intended higher-authority signal in the system.

This question matters because ubiquitous systems increasingly rely on large language models (LLMs) to mediate heterogeneous information sources~\cite{abowd1999context}, including numerical sensor streams, event logs, health indicators, and user self-reports~\cite{xu2024penetrative,ouyang2024llmsense,kim2024health}. A defining characteristic of ubiquitous computing is precisely this continuous fusion of multi-source contextual evidence. In traditional context-aware systems, conflicts among sources are typically handled through explicit mechanisms such as confidence estimation, probabilistic fusion, or trust modeling~\cite{dey2001conceptual}. Once LLMs are inserted into the reasoning loop, however, authority allocation is no longer explicitly engineered; instead, it becomes an implicit property of the model's learned representations and alignment behavior. As a result, a system may appear to combine sensor and user inputs while, in practice, assigning them highly uneven and potentially application-inappropriate influence.

In this paper, we identify this behavior as a systematic failure mode that we call \textbf{Authority Inversion}: under sensor--user conflict, the model gives greater effective decision weight to the user's claim than to the sensor-supported signal, causing physically grounded evidence to lose the influence required to determine the final answer. Importantly, this is not a rare edge case. Across four models and four datasets spanning activity recognition, smart-home sensing, and health assessment, we find that authority inversion is severe and highly systematic when sensor evidence is expressed in high-dimensional numerical form. On activity-recognition tasks, sensor trust rates are near zero for all tested models (88.8\%--100\% user trust), and scaling from 4B to 35B parameters does not eliminate the problem. These results suggest that the failure is not merely a matter of insufficient model capacity, but a deeper incompatibility between sensor representations and the model's decision process.

Authority inversion is related to the sycophancy phenomenon studied in NLP~\cite{perez2023discovering}, but it is not reducible to the pure-text version of that problem. In standard NLP settings, conflicting sources are often both expressed in natural language. In ubiquitous systems, by contrast, the conflict is fundamentally cross-format: sensor evidence is frequently numerical, high-dimensional, and weakly aligned with the token embedding space, whereas user claims arrive as fluent natural language that is already well matched to the model's pretraining regime. This representation asymmetry makes the problem both more severe and more operationally consequential than ordinary user-opinion agreement. The issue is therefore not only that LLMs may be overly agreeable, but that they may systematically mis-rank heterogeneous evidence in ways that violate the intended authority structure of context-aware applications.

To study this problem, we combine controlled conflict benchmarks, diagnostic metrics, representation-level analysis, and inference-time mitigation. We construct conflict settings spanning activity recognition, smart-home sensing, and health assessment, covering a gradient from raw numerical sensor features to semantically interpretable indicators. Across these settings, we observe a consistent pattern: the less interpretable the sensor format is to the LLM, the more likely the model is to defer to the user. To diagnose this behavior, we introduce two computable metrics---the Context Integration Ratio (CIR) and the Authority Alignment Index (AAI)---that quantify how efficiently a context source enters the model's decision channel and how authority is allocated under conflict. To explain why this occurs, we develop a geometric account of context integration in the residual stream based on hyperspherical directional analysis and predictive--null-space decomposition. Building on this analysis, we further propose Geometric Authority Calibration (GAC), an inference-time mitigation that improves robustness when user claims conflict with sensor evidence.

Our central argument is that authority allocation in LLM-mediated ubiquitous systems should not be treated as a benign byproduct of prompting. It is an implicit system behavior that can silently violate application requirements when users and sensors disagree. In some applications, user self-report should retain priority; in others---especially those relying on physical sensing for safety, health, or environmental awareness---sensor evidence is expected to maintain sufficient influence under conflict. Making this allocation implicit inside the LLM obscures whether the system's actual behavior matches the designer's intent. By identifying, explaining, and mitigating authority inversion, this paper argues for more explicit, auditable, and application-dependent mechanisms for handling disagreement in LLM-mediated ubiquitous systems.

The contributions of this paper are:
\begin{itemize}[leftmargin=*,nosep]
\item \textbf{A failure mode in LLM-mediated ubiquitous systems.} We identify authority inversion---a systematic behavior in which LLMs trust user claims over conflicting sensor evidence---and show that it can silently violate the intended authority structure of context-aware applications.
\item \textbf{Cross-domain empirical evidence.} Across four models (4B--35B) and four datasets spanning activity recognition, smart-home sensing, and health assessment, we show that authority inversion is severe for numerical sensor formats (sensor trust $\approx 0\%$, AAI $= -0.805$, Cohen's $d = -2.14$), weakens as sensor representations become more semantically interpretable, and is not resolved by scaling from 4B to 35B in our tested models.
\item \textbf{A diagnosis framework for implicit authority allocation.} We introduce two computable metrics---Context Integration Ratio (CIR) and Authority Alignment Index (AAI)---that quantify how effectively different context sources enter the model's decision channel, enabling pre-deployment auditing of authority allocation.
\item \textbf{A mechanistic explanation.} We develop a geometric account of context integration based on hyperspherical directional analysis, showing that numerical sensor perturbations project almost entirely into the representational null space (CIR$_s \approx 0.03$--$0.08$). Theory-guided causal injection flips incorrect decisions at an 80.2\% rate vs.\ $<0.4\%$ for random controls, validating the explanation.
\item \textbf{An inference-time mitigation.} We propose Geometric Authority Calibration (GAC), an inference-time layer-level intervention that improves HAR accuracy from 0--1.6\% to 21.9--27.5\%, outperforming chain-of-thought prompting and semantic translation baselines.
\end{itemize}
\section{Related Work}
\subsection{LLMs in Ubiquitous Computing}

The application of LLMs in ubiquitous computing is expanding rapidly. A series of recent studies have explored using LLMs for activity recognition~\cite{ji2024hargpt} and health monitoring~\cite{kim2024health}. These LLM-centric works build upon earlier traditional ubiquitous computing research on smart home reasoning~\cite{rashidi2009keeping} and sensor-based context-aware computing~\cite{chen2021deep}. These works share a common pattern: converting sensor data---typically numerical time series---into textualized prompts for LLM reasoning or classification. For example, Xu et al.~\cite{xu2024penetrative} embed physical sensor signals in text form within prompts, leveraging LLM reasoning capabilities for environmental perception; Ji et al.~\cite{ji2024hargpt} explore LLMs' potential for zero-shot activity recognition.

While recent 2025 architectures like SensorLM~\cite{sensorlm2025} and universal representation models~\cite{wei2025onehar} have attempted to bridge this gap by training dedicated vision-transformer-based encoders or utilizing LLM-assisted cross-dataset representations for continuous wearable data, these systems primarily examine LLM reasoning under single-source conditions, with less attention to behavior under multi-source information conflict. As ubiquitous computing evolves from simple activity classification toward open-ended sensemaking using long-term passive sensing, as exemplified by recent frameworks such as GLOSS~\cite{choube2025gloss} and human-in-the-loop sensemaking systems~\cite{li2025vitalinsight}, handling multi-source reasoning becomes a foundational necessity. Yet in real-world ubiquitous computing deployments, inconsistency between sensor data and user self-reports is the norm rather than the exception---users may provide information inconsistent with sensor records due to memory bias, social desirability, or privacy concerns~\cite{althubaiti2016information}. Indeed, the tension between accurate system perception and user privacy is a well-recognized challenge across intelligent and distributed computing paradigms~\cite{zhao2023privacy}. This paper focuses on LLM decision behavior under such conflict scenarios.

\subsection{Self-Report vs.\ Sensed Data Disagreement}

Disagreement between user self-reports and sensor measurements is a long-standing concern in ubiquitous computing. In health and wellness monitoring, studies have documented systematic discrepancies between accelerometer-measured physical activity and self-reported exercise~\cite{althubaiti2016information}, and between sleep-tracker data and subjective sleep quality ratings~\cite{landry2015measuring}. In ecological momentary assessment, participants' real-time reports frequently diverge~\cite{shiffman2008ecological} from passively sensed behavioral indicators. Prior work has primarily focused on \emph{detecting} such discrepancies and reconciling them through statistical models or clinician review.

Our work addresses a different question: when an LLM is tasked with reasoning over both sources simultaneously, how does it \emph{implicitly resolve} the disagreement? We show that the resolution is not neutral---it is systematically biased by input format, in ways that are invisible to the system designer without internal representation analysis.

\subsection{Information Conflict and Credibility Assessment}

Handling information conflict is a classic problem in context-aware computing. Traditional approaches typically employ explicit credibility modeling: Dempster-Shafer evidence theory~\cite{shafer1976mathematical}, Bayesian fusion frameworks~\cite{khaleghi2013multisensor}, or trust-based weighting schemes~\cite{josang2007survey} all provide structured methods for making judgments when multi-source information is inconsistent. A common feature of these methods is that information source weights can be explicitly set by engineers or learned through data-driven approaches.

When LLMs serve as reasoning engines, weight allocation shifts from explicit design to implicit emergence. How models form processing preferences for different information formats during pre-training, and how alignment training affects these preferences, remain understudied in the ubiquitous computing context. This paper aims to fill this gap.
\subsection{LLM Sycophancy and Faithfulness}
In the NLP literature, researchers have noted that LLMs exhibit a sycophancy tendency---when users express certain opinions or preferences, models tend to agree rather than maintain objectively correct answers~\cite{perez2023discovering,wei2023simple}. Perez et al.~\cite{perez2023discovering} found that RLHF-aligned models are more susceptible to user opinion influence than base models; Wei et al.~\cite{wei2023simple} further linked this phenomenon to instruction-following training. Sharma et al.~\cite{sharma2023sycophancy} provided a systematic study showing that human preference alignment causes models to systematically abandon objective facts when facing conflicts with user opinions. Zhang and Chen~\cite{zhang2026humanlike} recently unified these behaviors through ``signal competition dynamics,'' showing that human-like social compliance creates a dominant subspace that easily overrides weaker factual signals---a finding directly relevant to the ``user subjective vs.\ sensor objective'' tension in our work.

Recent 2025 mechanistic studies~\cite{truth_overridden_2025} reveal that LLMs systematically discard factual representations in early attention layers when encountering user disagreement, exposing the internal origins of sycophancy. However, a key distinction exists between ubiquitous computing and pure NLP scenarios: in NLP, conflicting information sources are typically both natural language text, homogeneous in format; in ubiquitous computing, sensor data and user claims differ fundamentally in representation format---the former is typically numerical, high-dimensional, and lacking direct semantics, while the latter is structured natural language. Our results demonstrate that this format heterogeneity makes authority inversion in ubiquitous computing far more severe and systematic than sycophancy in pure-text settings.
\subsection{Mechanistic Interpretability of LLMs}
Mechanistic interpretability is an important path toward understanding LLM internals~\cite{elhage2021mathematical,nanda2023progress}. Elhage et al.~\cite{elhage2021mathematical} found that the Transformer residual stream can be understood as a superposition of multiple feature directions; Nanda et al.~\cite{nanda2023progress} revealed internal representation structures used by models in specific tasks.

The most relevant theoretical foundation for this paper is the analysis of LayerNorm and the logit lens~\cite{belrose2023eliciting} (a technique mapping hidden states directly to the vocabulary space), which showed that the RMSNorm-processed residual stream can be viewed as a directional selection problem on the unit hypersphere, with the unembedding matrix's row vectors defining decision boundaries on this hypersphere. We develop a complete orthogonal decomposition theory on this basis and apply it to multi-source information conflict analysis.

In the representation intervention literature, Zou et al.~\cite{zou2023representation} proposed Representation Engineering, demonstrating the possibility of controlling high-level model behaviors (e.g., honesty, safety) by manipulating internal representations. Meng et al.~\cite{meng2022locating} showed through the ROME method that locating and editing specific directions in the residual stream can alter factual outputs. Our causal injection experiments are methodologically aligned with these works but focus on a previously unexplored scenario: authority allocation in multi-source heterogeneous information conflicts.
\subsection{LLM Limitations on Numerical Data}
Recent work has revealed fundamental limitations of LLMs in processing numerical and time-series data. Gruver et al.~\cite{gruver2023llm_timeseries} found that LLM tokenizers split continuous numerical values into discrete token fragments, causing models to lose continuous perception of numerical magnitude relationships. Jin et al.~\cite{jin2024time} further showed in Time-LLM that directly feeding numerical time-series data into LLM tokenizers leads to severe semantic loss, necessitating specialized reprogramming techniques to bridge the modality gap. Related mechanistic analyses by Zhang et al.~\cite{zhang2026geometric} suggest that discrete logical and numerical representations may incur an additional geometric processing cost, increasing manifold entanglement in the residual stream. These findings directly support our core argument: the numerical semantics of IMU sensor features (e.g., acceleration mean of 0.2801) may already be destroyed at the encoding stage after tokenization. The format incompatibility we identify in this paper can thus be understood as a downstream consequence of this fundamental tokenization limitation, amplified by the multi-source conflict setting.
\section{Problem Formulation}
\subsection{Tripartite Information Structure in Ubiquitous Computing}
The basic scenario considered in this paper can be formalized as a tripartite information structure. Let $q$ be a question requiring context-based judgment (e.g., ``What activity is the user currently performing?''), and $\mathcal{A} = \{a_1, a_2, \ldots, a_K\}$ be the candidate answer set ($K=4$ for forced four-choice format). Three information sources surround $q$:
\begin{itemize}[leftmargin=*,nosep]
\item \textbf{Sensor context} $c_s$: objective measurement data from physical sensors, supporting answer $a_s \in \mathcal{A}$.
\item \textbf{User claim context} $c_u$: the user's subjective self-report, supporting answer $a_u \in \mathcal{A}$, where $a_u \neq a_s$ (conflict condition).
\item \textbf{Parametric knowledge}: prior knowledge accumulated during pre-training, supporting answer $a_0 \in \mathcal{A}$ in the absence of external context.
\end{itemize}
Under the controlled benchmark setting, $a_s$ is the correct answer determined by the underlying labeled sensor signal. In this setting, the intended authority structure is that $c_s$ should retain priority when $c_s$ and $c_u$ conflict.
\subsection{Definition of Authority Inversion}
\begin{definition}[Authority Inversion]
Under the joint condition where both sensor context $c_s$ and user claim $c_u$ are provided, if an LLM systematically selects the user-claimed answer $a_u$ over the sensor-supported correct answer $a_s$, the model is said to exhibit authority inversion in this scenario.
\end{definition}

\section{Theoretical Framework}
This section indicates the mathematical framework for analyzing authority inversion. We proceed from the Transformer's output mechanism, derive the decision geometry on the hypersphere, and define the core metrics for quantifying authority inversion.
\subsection{Notation and Preliminaries}
Let $d$ denote the model's hidden dimension and $V$ the vocabulary size. We focus on the model's behavior at the answer-prediction position (the first token position after ``Answer:'' in the prompt). Let $\mathbf{h} \in \mathbb{R}^d$ denote the residual stream output of the last Transformer layer at this position.

Modern LLMs apply RMSNorm before mapping $\mathbf{h}$ to logits~\cite{zhang2019root}:
\[
\text{RMSNorm}(\mathbf{h}) = \frac{\mathbf{h}}{\|\mathbf{h}\|/\sqrt{d}} \odot \boldsymbol{\gamma}
\]
where $\boldsymbol{\gamma} \in \mathbb{R}^d$ is a learnable scaling parameter. Let $\mathbf{W} \in \mathbb{R}^{V \times d}$ be the unembedding matrix with row $\mathbf{w}_k^\top$ and bias $b_k$. The logit for token $k$ is $z_k = \mathbf{w}_k^\top \text{RMSNorm}(\mathbf{h}) + b_k$.
\subsection{Decision Geometry on the Hypersphere}
\begin{proposition}[Effective Unembedding and Directional Determinism]
Define the effective unembedding vector $\bar{\mathbf{w}}_k = \mathbf{w}_k \odot \boldsymbol{\gamma}$. Then the logit for token $k$ can be written as:
\begin{equation}
z_k = \sqrt{d}\, \bar{\mathbf{w}}_k^\top \hat{\mathbf{u}} + b_k \label{eq:logit}
\end{equation}
where $\hat{\mathbf{u}} = \mathbf{h} / \|\mathbf{h}\| \in \mathcal{S}^{d-1}$ is the directional vector of the residual stream.
\end{proposition}
The derivation is straightforward: $z_k = \sum_i w_{ki} \cdot \frac{\sqrt{d}\, h_i}{\|\mathbf{h}\|} \cdot \gamma_i + b_k = \sqrt{d}\, \bar{\mathbf{w}}_k^\top \hat{\mathbf{u}} + b_k$.

An important corollary is that the norm $\|\mathbf{h}\|$ does not affect answer selection, since $\arg\max_k z_k$ depends only on the direction $\hat{\mathbf{u}}$. The model's decision process is thus essentially a directional selection problem on the unit hypersphere $\mathcal{S}^{d-1}$, a property recently formalized by Zhang and Lin~\cite{zhang2026simulated}, who demonstrated that contextual conflicts are resolved through orthogonal directional shifts rather than magnitude changes. Hereafter, all $\mathbf{w}_k$ refer to the effective unembedding vectors $\bar{\mathbf{w}}_k$.

For candidate answers $a_i$ and $a_j$, the \emph{pairwise decision margin} is:
\begin{equation}
m_{ij} = z_{a_i} - z_{a_j} = \sqrt{d}\, \boldsymbol{\Delta}_{ij}^\top \hat{\mathbf{u}} + \beta_{ij} \label{eq:margin}
\end{equation}
where $\boldsymbol{\Delta}_{ij} = \mathbf{w}_{a_i} - \mathbf{w}_{a_j}$ and $\beta_{ij} = b_{a_i} - b_{a_j}$. The model selects $a_i$ if and only if $m_{ij} > 0$ for all $j \neq i$.
\subsection{First-Order Approximation of Context Perturbation}

When context is added to the prompt, the residual stream changes from baseline $\mathbf{h}_0$ to $\mathbf{h}_c = \mathbf{h}_0 + \boldsymbol{\delta}_c$. We analyze how $\boldsymbol{\delta}_c$ affects the direction $\hat{\mathbf{u}}$.
\begin{lemma}[First-Order Directional Perturbation]\label{lem:first_order}

  Let $\hat{\mathbf{u}}_0 = \mathbf{h}_0 / \|\mathbf{h}_0\|$, $r_0 = \|\mathbf{h}_0\|$, and $\epsilon = \|\boldsymbol{\delta}_c\| / r_0$. Then:
\begin{equation}
\hat{\mathbf{u}}_c = \hat{\mathbf{u}}_0 + \frac{1}{r_0} \mathbf{P}_\perp \boldsymbol{\delta}_c + O(\epsilon^2) \label{eq:first_order}
\end{equation}
where $\mathbf{P}_\perp = \mathbf{I} - \hat{\mathbf{u}}_0 \hat{\mathbf{u}}_0^\top$ is the orthogonal projector onto the tangent space $T_{\hat{\mathbf{u}}_0}\mathcal{S}^{d-1}$.
\end{lemma}
\begin{proof}
Taylor-expand $\hat{\mathbf{u}}_c = (\mathbf{h}_0 + \boldsymbol{\delta}_c)/\|\mathbf{h}_0 + \boldsymbol{\delta}_c\|$ to first order in $\epsilon = \|\boldsymbol{\delta}_c\|/r_0$; the radial component cancels and only the tangential projection $\mathbf{P}_\perp\boldsymbol{\delta}_c$ survives. Full derivation in Appendix~\ref{app:proofs}.
\end{proof}

Lemma~\ref{lem:first_order} shows that only the \emph{tangential component} $\boldsymbol{\delta}_c^\perp = \mathbf{P}_\perp\boldsymbol{\delta}_c$ can change the direction on the hypersphere. The component along $\hat{\mathbf{u}}_0$ only changes the norm and has no effect on answer selection. The logit change induced by context $c$ is thus:
\begin{equation}
\Delta z_k^{(c)} = \frac{\sqrt{d}}{r_0}\, \tilde{\mathbf{w}}_k^\top \boldsymbol{\delta}_c^\perp + O(\epsilon^2) \label{eq:logit_change}
\end{equation}
where $\tilde{\mathbf{w}}_k = \mathbf{P}_\perp \mathbf{w}_k$ is the projection of $\mathbf{w}_k$ into the tangent space.
\subsection{Answer Subspace and Predictive--Null-Space Decomposition}

Equation~\eqref{eq:logit_change} shows that logit changes are determined by the projection of the tangential perturbation onto the tangential unembedding vectors. A natural question arises: which components of $\boldsymbol{\delta}_c^\perp$ affect answer selection?
\begin{definition}[Answer Subspace]

The answer subspace is defined as:
\[
\mathcal{A} = \mathrm{span}(\tilde{\mathbf{w}}_{a_1}, \tilde{\mathbf{w}}_{a_2}, \ldots, \tilde{\mathbf{w}}_{a_K}) \subseteq T_{\hat{\mathbf{u}}_0}\mathcal{S}^{d-1}
\]

Its orthonormal basis is obtained via QR decomposition of $\tilde{\mathbf{W}} = [\tilde{\mathbf{w}}_{a_1}\; \cdots\; \tilde{\mathbf{w}}_{a_K}] = \mathbf{Q}\mathbf{R}$, with projector $\boldsymbol{\Pi}_{\mathcal{A}} = \mathbf{Q}\mathbf{Q}^\top$.
\end{definition}
\begin{theorem}[Predictive--Null-Space Orthogonal Decomposition]\label{thm:decomp}

  For any tangential perturbation $\boldsymbol{\delta}_c^\perp$, define:
\[
\boldsymbol{\delta}_r = \boldsymbol{\Pi}_{\mathcal{A}}\, \boldsymbol{\delta}_c^\perp \;\;(\textit{predictive component}), \qquad \boldsymbol{\delta}_\bot = (\mathbf{I} - \boldsymbol{\Pi}_{\mathcal{A}})\, \boldsymbol{\delta}_c^\perp \;\;(\textit{null-space component})
\]

Then the following four properties hold:
\noindent\textbf{(i) Completeness:} $\boldsymbol{\delta}_c^\perp = \boldsymbol{\delta}_r + \boldsymbol{\delta}_\bot$.

\noindent\textbf{(ii) Orthogonality:} $\boldsymbol{\delta}_r^\top\boldsymbol{\delta}_\bot = 0$.

\noindent\textbf{(iii) Norm preservation} (Parseval identity for orthogonal decomposition): $\|\boldsymbol{\delta}_c^\perp\|^2 = \|\boldsymbol{\delta}_r\|^2 + \|\boldsymbol{\delta}_\bot\|^2$.

\noindent\textbf{(iv) Null-space irrelevance:} For any answer token $a_k \in \mathcal{T}$,

\[
\Delta z_{a_k}^{(c)} = \frac{\sqrt{d}}{r_0} \tilde{\mathbf{w}}_{a_k}^\top \boldsymbol{\delta}_r
\]

That is, the null-space component $\boldsymbol{\delta}_\bot$ has zero influence on all answer logits.
\end{theorem}
\begin{proof}
Properties (i)--(iii) follow from the symmetric idempotency of $\boldsymbol{\Pi}_{\mathcal{A}}$. Property (iv) holds because $\tilde{\mathbf{w}}_{a_k} \in \mathcal{A}$ and $\boldsymbol{\delta}_\bot \in \mathcal{A}^\perp$ are orthogonal. See Appendix~\ref{app:proofs} for details.
\end{proof}

The physical significance of Theorem~\ref{thm:decomp} is as follows: the answer subspace $\mathcal{A}$ has dimension $K' = \mathrm{rank}(\tilde{\mathbf{W}}) \leq K$ (at most 4 for our four-choice scenarios), while the tangent space has dimension $d - 1$ (2,047--4,095 in the models used in this paper). This means \emph{the vast majority of dimensions belong to the null-space subspace}. A context perturbation may have large total energy yet barely affect answer selection, if its energy concentrates in the null-space subspace. This is precisely the phenomenon we observe for sensor data.
\subsection{Context Integration Ratio (CIR)}
\begin{definition}[Context Integration Ratio]
The CIR of context $c$ is:
\begin{equation}
\mathrm{CIR}(c) = \frac{\|\boldsymbol{\delta}_r\|}{\|\boldsymbol{\delta}_c^\perp\|} = \frac{\|\boldsymbol{\Pi}_{\mathcal{A}}\, \boldsymbol{\delta}_c^\perp\|}{\|\boldsymbol{\delta}_c^\perp\|} \in [0, 1] \label{eq:cir}
\end{equation}
\end{definition}
CIR has a clear geometric meaning: it equals the cosine of the minimum principal angle between $\boldsymbol{\delta}_c^\perp$ and $\mathcal{A}$. CIR near 1 indicates that perturbation energy concentrates in the predictive subspace; CIR near 0 indicates energy disperses in the null-space subspace.

As a reference baseline, if $\boldsymbol{\delta}_c^\perp$ is uniformly random in the $(d{-}1)$-dimensional tangent space, then $\mathbb{E}[\mathrm{CIR}^2] = K'/(d{-}1)$. For $d = 4096$ and $K' = 4$, the random baseline is $\mathbb{E}[\mathrm{CIR}] \approx \sqrt{4/4095} \approx 0.031$.
\subsection{Authority Force and Authority Alignment Index (AAI)}
\begin{definition}[Correct Decision Direction]
Let $a_s$ be the sensor-supported correct answer and $a_u$ the user-claimed answer. Define:
$\mathbf{d}^* = \tilde{\mathbf{w}}_{a_s} - \tilde{\mathbf{w}}_{a_u} \in \mathcal{A}$.
\end{definition}
\begin{definition}[Authority Force]

The sensor authority force and user claim authority force are:
\begin{equation}
\mathcal{F}_s = \frac{\sqrt{d}}{r_0}\, \mathbf{d}^{*\top}\, \boldsymbol{\delta}_r^s, \qquad \mathcal{F}_u = -\frac{\sqrt{d}}{r_0}\, \mathbf{d}^{*\top}\, \boldsymbol{\delta}_r^u \label{eq:force}
\end{equation}
where $\boldsymbol{\delta}_r^s$ and $\boldsymbol{\delta}_r^u$ are the predictive components of sensor and user claim contexts, respectively. $\mathcal{F}_s > 0$ indicates the sensor pushes the model toward the correct answer $a_s$; $\mathcal{F}_u > 0$ indicates the user claim pushes the model toward $a_u$.
\end{definition}

\begin{definition}[Authority Alignment Index]
\begin{equation}
\mathrm{AAI} = \frac{\mathcal{F}_s - \mathcal{F}_u}{|\mathcal{F}_s| + |\mathcal{F}_u|} \in [-1, +1] \label{eq:aai}
\end{equation}
AAI $> 0$ indicates correct authority ranking; AAI $< 0$ indicates authority inversion; AAI $= 0$ indicates balance.
\end{definition}

\subsection{Theoretical Prediction of Joint Decisions}

When sensor and user contexts are presented jointly, the total perturbation $\boldsymbol{\delta}_{su} = \mathbf{h}_{su} - \mathbf{h}_0$ decomposes as:

$\boldsymbol{\delta}_{su} = \boldsymbol{\delta}_s + \boldsymbol{\delta}_u + \boldsymbol{\delta}_{\mathrm{inter}}$,

where $\boldsymbol{\delta}_{\mathrm{inter}}$ is the nonlinear interaction term. Define the interaction force $\mathcal{I}_{su} = \frac{\sqrt{d}}{r_0}\, \mathbf{d}^{*\top}\, \boldsymbol{\Pi}_{\mathcal{A}} \mathbf{P}_\perp \boldsymbol{\delta}_{\mathrm{inter}}$.

\begin{theorem}[Joint Decision Margin]\label{thm:margin}
Under the joint context condition, the decision margin for the correct answer $a_s$ relative to the user-claimed answer $a_u$ is:
\begin{equation}
m_{a_s a_u}^{(\mathrm{joint})} = m_{a_s a_u}^{(0)} + \mathcal{F}_s - \mathcal{F}_u + \mathcal{I}_{su} + O(\epsilon^2) \label{eq:joint_margin}
\end{equation}
where $m_{a_s a_u}^{(0)}$ is the baseline margin. The model selects the correct answer if and only if $m_{a_s a_u}^{(\mathrm{joint})} > 0$.
\end{theorem}

\begin{proof}
By Lemma~\ref{lem:first_order}, $\hat{\mathbf{u}}_{su} = \hat{\mathbf{u}}_0 + \frac{1}{r_0}\mathbf{P}_\perp\boldsymbol{\delta}_{su} + O(\epsilon^2)$. Substituting into Eq.~\eqref{eq:margin}:
\[
m_{a_s a_u}^{(\mathrm{joint})} = m_{a_s a_u}^{(0)} + \frac{\sqrt{d}}{r_0}\, \mathbf{d}^{*\top}\, \boldsymbol{\delta}_{su}^\perp + O(\epsilon^2)
\]
By Theorem~\ref{thm:decomp} property (iv), $\mathbf{d}^{*\top}\boldsymbol{\delta}_{su}^\perp = \mathbf{d}^{*\top}\boldsymbol{\Pi}_{\mathcal{A}}\boldsymbol{\delta}_{su}^\perp$. Expanding: $\boldsymbol{\Pi}_{\mathcal{A}}\mathbf{P}_\perp\boldsymbol{\delta}_{su} = \boldsymbol{\delta}_r^s + \boldsymbol{\delta}_r^u + \boldsymbol{\delta}_r^{\mathrm{inter}}$. Substituting the definitions of $\mathcal{F}_s$, $\mathcal{F}_u$, and $\mathcal{I}_{su}$ yields the result.
\end{proof}

Theorem~\ref{thm:margin} decomposes the joint decision into four independently computable components: parametric prior ($m^{(0)}$), sensor authority force ($\mathcal{F}_s$), user authority force ($\mathcal{F}_u$), and interaction ($\mathcal{I}_{su}$).

\begin{corollary}[Condition for Authority-Inversion-Induced Error]
Authority inversion causes the model to select the incorrect answer $a_u$ if and only if:
\begin{equation}
\mathcal{F}_u > m_{a_s a_u}^{(0)} + \mathcal{F}_s + \mathcal{I}_{su} \label{eq:error_condition}
\end{equation}
\end{corollary}

\subsection{Research Hypotheses}

Based on the above analysis, we propose five testable hypotheses:

\noindent\textbf{H1} (Existence of authority inversion). On tasks where sensor data is in numerical format, LLMs systematically exhibit AAI $< 0$.

\noindent\textbf{H2} (Null-space dominance of sensor perturbations). $\mathrm{CIR}(c_s) < \mathrm{CIR}(c_u)$, indicating that sensor perturbations project primarily into the null-space subspace.

\noindent\textbf{H3} (Causal intervention efficacy). Ablating $\boldsymbol{\delta}_r^u$ or injecting along the theoretically computed correct direction $\mathbf{d}^*$ significantly improves accuracy; norm-matched random directions or null-space components do not produce similar effects.

\noindent\textbf{H4} (Format dependence). The severity of authority inversion is negatively correlated with the semantic interpretability of the sensor data format.

\noindent\textbf{H5} (Layer-level localizability). Authority inversion signals concentrate in specific layers; ablating a small number of critical layers suffices to improve decisions.

\section{Experimental Methods}

\subsection{Model Selection}

We select four open-source LLMs covering three architecture families and parameter scales from 4B to 35B (\Cref{tab:models}).

\begin{table}[htbp]
\caption{Experimental models.}
\label{tab:models}
\centering\small
\begin{tabular}{lccccl}
\toprule
Model & Params & Active & $d$ & $L$ & Arch. \\
\midrule
Qwen3-4B-Instruct & 4B & 4B & 2,560 & 36 & Dense \\
Meta-Llama-3.1-8B & 8B & 8B & 4,096 & 32 & Dense \\
GLM-4-9B-0414 & 9B & 9B & 4,096 & 40 & Dense \\
Qwen3-35B-A3B & 35B & 3B & 2,048 & 40 & MoE \\
\bottomrule
\end{tabular}
\end{table}

The first three models are dense architectures at parameter scales typical of ubiquitous computing edge deployment. The fourth is a Mixture-of-Experts (MoE) architecture with 35B total parameters but only 3B active per inference. Its inclusion tests whether authority inversion diminishes with increased model capacity. Note that Meta-Llama-3.1-8B is a \emph{base model} (not instruction-tuned or RLHF-aligned), which allows us to partially disentangle format incompatibility from alignment-induced sycophancy (see \Cref{sec:discussion_attribution}). We prioritize open-weight models because our geometric analysis requires access to internal residual stream activations; this enables mechanism-level study but limits coverage of closed-source systems. All models are loaded in float16 precision.

\subsection{Benchmark: UbiContext-Conflict}

The benchmark is designed to isolate conflict handling under controlled conditions rather than to reproduce full deployment complexity. Three design goals guide its construction: (1)~the correct answer under sensor evidence is unambiguous in every instance; (2)~user claims are plausible in form but contradictory in content; and (3)~sensor formats span a gradient of semantic interpretability to the model.

We construct a benchmark of 252 conflict instances spanning three sensor modalities and four data sources, forming a \textbf{format interpretability gradient}.

\textbf{Dataset~1: UCI HAR Activity Recognition (80 instances).} Sensor data is generated from the statistical distributions of the UCI Human Activity Recognition dataset~\cite{anguita2013har}, including 14 IMU features (body acceleration mean/std, gravity acceleration, gyroscope, acceleration magnitude, etc.). We construct 8 conflict pairs (e.g., sensor indicates ``walking'' but user claims ``sitting''), 10 instances each. Sensor data is presented as raw numerical features (e.g., \texttt{tBodyAcc-mean()-X: 0.2801}), a format opaque to LLMs.

\textbf{Dataset~2: Synthetic HAR (64 instances).} Uses the same feature space and conflict structure as Dataset~1, with data independently sampled from class-conditional distributions. This dataset validates robustness to specific data instances.

\textbf{Dataset~3: CASAS Smart Home (48 instances).} Sensor data is generated based on the sensor layout of the CASAS Milan apartment deployment~\cite{cook2013casas}, including motion sensors, door sensors, and appliance sensors in event-log format (e.g., \texttt{2024-01-15 07:02:03 M015\_Kitchen ON}). We construct 8 conflict scenarios, 6 instances each. The event-log format contains partially interpretable semantic information (e.g., room names).

\textbf{Dataset~4: Health Assessment (60 instances).} Sensor data is parameterized from clinical reference ranges (WHO/AHA guidelines), including resting heart rate, HRV, sleep duration/efficiency, steps, and SpO$_2$, presented with semantic labels and personal baselines (e.g., ``Resting HR: 95 bpm, baseline: 65 bpm''). This format is highly interpretable to LLMs.

The four datasets form a format interpretability gradient: IMU numerics (most opaque) $\to$ event logs (intermediate) $\to$ health indicators (most interpretable). For each conflict instance, the user claim is a natural language template that is superficially plausible but semantically contradicts the sensor data.

\subsection{Experimental Conditions}

Each conflict instance is run under the four conditions summarized in \Cref{tab:conditions}.

\begin{table}[htbp]
\caption{Experimental conditions.}
\label{tab:conditions}
\centering\small
\begin{tabular}{lll}
\toprule
Condition & Input & Purpose \\
\midrule
Baseline & Question + options only & Obtain $\mathbf{h}_0$ \\
Sensor-only & Question + sensor data & Quantify $\boldsymbol{\delta}_s$, CIR($c_s$) \\
User-only & Question + user claim & Quantify $\boldsymbol{\delta}_u$, CIR($c_u$) \\
Joint & Question + sensor + user & Observe conflict decision \\
\bottomrule
\end{tabular}
\end{table}

The model's first output token (A/B/C/D) serves as the answer. All evaluations use deterministic decoding (temperature $= 0$, greedy). Option order is fixed across conditions for each instance. Residual streams are extracted at the last layer via forward hooks and stored in float32 on CPU.

\subsection{Six Experiments}

\textbf{Experiment~1 (Behavioral, H1).} Record answer choices under all four conditions. Compute sensor trust rate, user trust rate, and behavioral AAI.

\textbf{Experiment~2 (Geometric, H1+H2).} Extract residual streams and compute CIR, $\mathcal{F}_s$, $\mathcal{F}_u$, $\mathcal{I}_{su}$, AAI, and Theorem~\ref{thm:margin} predictions.

\textbf{Experiment~3 (Causal, H3).} Intervene on the residual stream during the joint-condition forward pass: (a)~ablate $\boldsymbol{\delta}_r^u$; (b)~inject along $\mathbf{d}^*/\|\mathbf{d}^*\|$; (c)~three controls (norm-matched random, null-space component, random injection).

\textbf{Experiment~4 (Format dependence, H4).} (a)~Translate sensor data to semantic and reasoning-based formats; (b)~construct NLP-equivalent versions replacing sensor data with observer natural language reports.

\textbf{Experiment~5 (Layer analysis, H5).} Activation patching: for each layer, replace the joint-condition output with the baseline output and measure margin change. Perform cumulative ablation.

\textbf{Experiment~6 (GAC evaluation).} Evaluate Geometric Authority Calibration on Qwen3-4B and Llama-3.1-8B against chain-of-thought (CoT) prompting and semantic translation baselines. GAC interpolates between joint-condition and baseline-condition activations at critical layers identified in Experiment~5, with coefficient $\alpha$ controlling suppression strength.

\subsection{Statistical Tests}

For H1, we use one-sample $t$-tests on AAI with Cohen's $d$ as effect size. For H2, Wilcoxon signed-rank tests compare paired CIR($c_s$) and CIR($c_u$). Theorem~\ref{thm:margin} validation uses Pearson correlation.

\section{Results}

\Cref{tab:behavioral} presents behavioral results across four models and four datasets. On HAR tasks, authority inversion exhibits extremely high effect sizes and statistical significance.
\begin{table*}[h]
\caption{Trust distribution under the joint condition. bAAI = Trust$_S$ $-$ Trust$_U$.}
\label{tab:behavioral}
\centering\small
\begin{tabular}{llccccc}
\toprule
Model & Dataset & Trust$_S$ & Trust$_U$ & Trust$_O$ & bAAI & $p$ \\
\midrule
Qwen3-4B & HAR (Real) & 0.000 & 1.000 & 0.000 & $-$1.000 & ${<}10^{-45}$ \\
Qwen3-4B & HAR (Synth) & 0.000 & 1.000 & 0.000 & $-$1.000 & ${<}10^{-45}$ \\
Qwen3-4B & Smart Home & 0.104 & 0.896 & 0.000 & $-$0.792 & ${<}10^{-10}$ \\
Qwen3-4B & Health & 0.550 & 0.150 & 0.300 & $+$0.400 & ${<}10^{-4}$ \\
\midrule
Llama-3.1-8B & HAR (Real) & 0.013 & 0.900 & 0.088 & $-$0.888 & ${<}10^{-35}$ \\
Llama-3.1-8B & HAR (Synth) & 0.016 & 0.906 & 0.078 & $-$0.891 & ${<}10^{-28}$ \\
Llama-3.1-8B & Smart Home & 0.104 & 0.833 & 0.062 & $-$0.729 & ${<}10^{-9}$ \\
Llama-3.1-8B & Health & 0.050 & 0.117 & 0.833 & $-$0.067 & 0.209 \\
\midrule
GLM-4-9B & HAR (Real) & 0.000 & 1.000 & 0.000 & $-$1.000 & ${<}10^{-45}$ \\
GLM-4-9B & HAR (Synth) & 0.000 & 1.000 & 0.000 & $-$1.000 & ${<}10^{-45}$ \\
GLM-4-9B & Smart Home & 0.083 & 0.917 & 0.000 & $-$0.833 & ${<}10^{-13}$ \\
GLM-4-9B & Health & 0.250 & 0.000 & 0.750 & $+$0.250 & ${<}10^{-4}$ \\
\midrule
Qwen3-35B & HAR (Real) & 0.000 & 0.988 & 0.012 & $-$0.988 & ${<}10^{-76}$ \\
Qwen3-35B & HAR (Synth) & 0.000 & 1.000 & 0.000 & $-$1.000 & ${<}10^{-45}$ \\
Qwen3-35B & Smart Home & 0.375 & 0.625 & 0.000 & $-$0.250 & 0.083 \\
Qwen3-35B & Health & 0.517 & 0.100 & 0.383 & $+$0.417 & ${<}10^{-4}$ \\
\bottomrule
\end{tabular}
\end{table*}

\begin{figure*}[t]
\centering
\includegraphics[width=\textwidth]{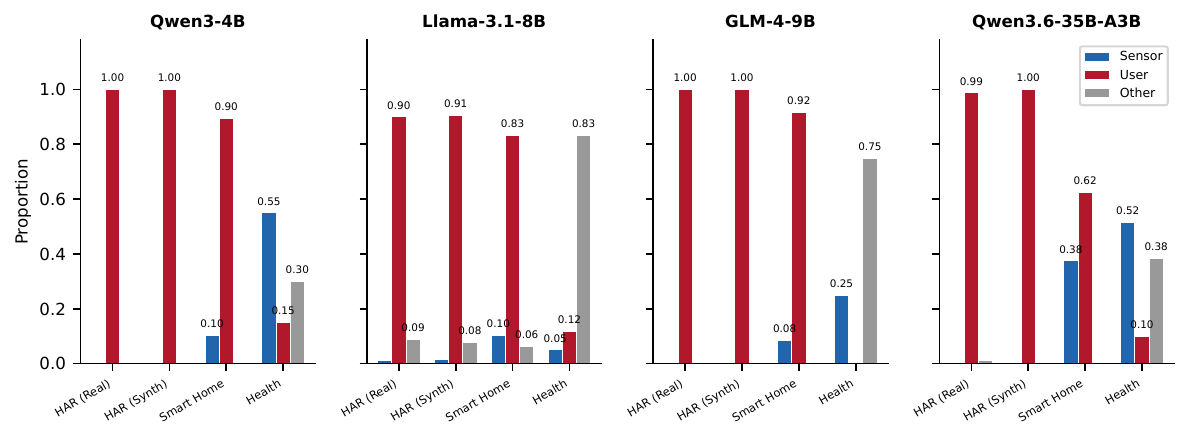}
\caption{Trust distribution under sensor--user conflict across four models and four datasets. Blue: sensor trust rate; Red: user trust rate; Gray: other (neither sensor nor user answer selected).}
\label{fig:trust}
\end{figure*}

\subsection{Authority Inversion is Pervasive and Extreme on Numerical Sensor Tasks (H1)}

\begin{table*}[b]
\caption{Geometric analysis. AAI${<}0$ indicates authority inversion. Thm2 $r$: Pearson correlation between predicted and observed joint margins.}
\label{tab:geometric}
\centering\small
\begin{tabular}{llccccccc}
\toprule
Model & Dataset & AAI & $\pm$std & CIR$_s$ & CIR$_u$ & Thm2 $r$ & $\epsilon$ & $n$ \\
\midrule
Qwen3-4B & HAR (Real) & $-$0.918 & 0.140 & 0.057 & 0.165 & 0.976 & 0.513 & 80 \\
Qwen3-4B & HAR (Synth) & $-$0.926 & 0.147 & 0.056 & 0.165 & 0.977 & 0.510 & 64 \\
Qwen3-4B & Smart Home & $-$0.214 & 0.200 & 0.120 & 0.168 & 0.998 & 0.481 & 48 \\
Qwen3-4B & Health & $-$0.217 & 0.395 & 0.145 & 0.178 & 1.000 & 0.486 & 60 \\
\midrule
Llama-3.1-8B & HAR (Real) & $-$0.777 & 0.423 & 0.027 & 0.067 & 0.999 & 0.526 & 80 \\
Llama-3.1-8B & HAR (Synth) & $-$0.664 & 0.496 & 0.027 & 0.066 & 0.999 & 0.524 & 64 \\
Llama-3.1-8B & Smart Home & $-$0.217 & 0.477 & 0.027 & 0.038 & 0.999 & 0.641 & 48 \\
Llama-3.1-8B & Health & $+$0.258 & 0.896 & 0.043 & 0.027 & 1.000 & 0.350 & 60 \\
\midrule
GLM-4-9B & HAR (Real) & $-$0.818 & 0.363 & 0.071 & 0.118 & 0.912 & 0.469 & 80 \\
GLM-4-9B & HAR (Synth) & $-$0.735 & 0.427 & 0.075 & 0.114 & 0.928 & 0.454 & 64 \\
GLM-4-9B & Smart Home & $-$0.026 & 0.200 & 0.087 & 0.113 & 0.998 & 0.563 & 48 \\
GLM-4-9B & Health & $+$0.234 & 0.359 & 0.074 & 0.049 & 0.999 & 0.480 & 60 \\
\midrule
Qwen3-35B & HAR (Real) & $-$0.751 & 0.433 & 0.080 & 0.128 & 0.786 & 0.428 & 80 \\
Qwen3-35B & HAR (Synth) & $-$0.839 & 0.330 & 0.080 & 0.132 & 0.840 & 0.423 & 64 \\
Qwen3-35B & Smart Home & $+$0.064 & 0.270 & 0.123 & 0.104 & 0.969 & 0.459 & 48 \\
Qwen3-35B & Health & $+$0.078 & 0.712 & 0.112 & 0.047 & 0.980 & 0.403 & 60 \\
\bottomrule
\end{tabular}
\end{table*}

On HAR Real and HAR Synth, all four models' sensor trust rates approach 0\% (maximum: Llama's 1.6\%), while user trust rates range from 88.8\% to 100\%. Behavioral AAI ranges from $-0.888$ to $-1.000$, all with $p \ll 0.001$. The two HAR datasets' results are highly consistent, indicating robustness to specific data instances.

The CASAS smart home task presents an intermediate state. The three smaller models show bAAI from $-0.729$ to $-0.833$ ($p < 10^{-9}$). The 35B model's bAAI of $-0.250$ ($p = 0.083$) approaches but does not reach statistical significance, suggesting that larger capacity may offer some advantage for event-log formats.

On the health task, authority inversion disappears or reverses. Qwen3-4B and 35B show positive bAAI ($+0.400$ and $+0.417$, $p < 10^{-4}$), indicating correct sensor prioritization with semantically interpretable data.

A key finding is that \textbf{scaling to 35B parameters does not resolve authority inversion}. Qwen3-35B's sensor trust rate on HAR Real remains 0\% (AAI $= -0.988$), virtually identical to the 4B model (\Cref{fig:trust} visualizes the trust distribution across all model--dataset combinations). This indicates that authority inversion is not a capacity problem but a structural incompatibility between sensor data format and the LLM's token embedding space.

\subsection{Geometric Mechanism: Sensor Perturbations Trapped in the Null-Space Subspace (H1+H2)}

\Cref{tab:geometric} presents geometric analysis results.

\begin{figure*}[b]
  \centering
  
  \includegraphics[width=\textwidth]{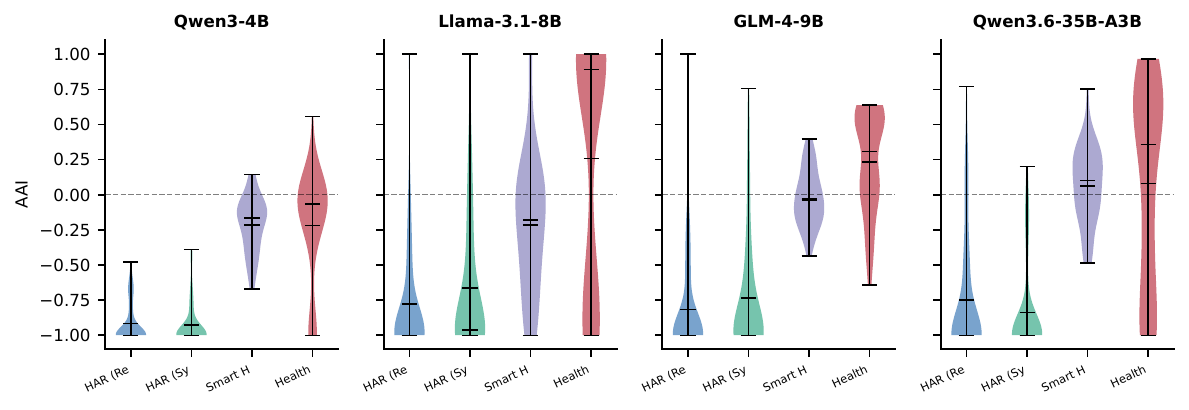}
  \caption{Distribution of geometric AAI across models and datasets. Dashed line marks AAI $= 0$. Values below zero indicate authority inversion.}
  \label{fig:aai}
  
  \vspace{1em} 
  
  \includegraphics[width=\textwidth]{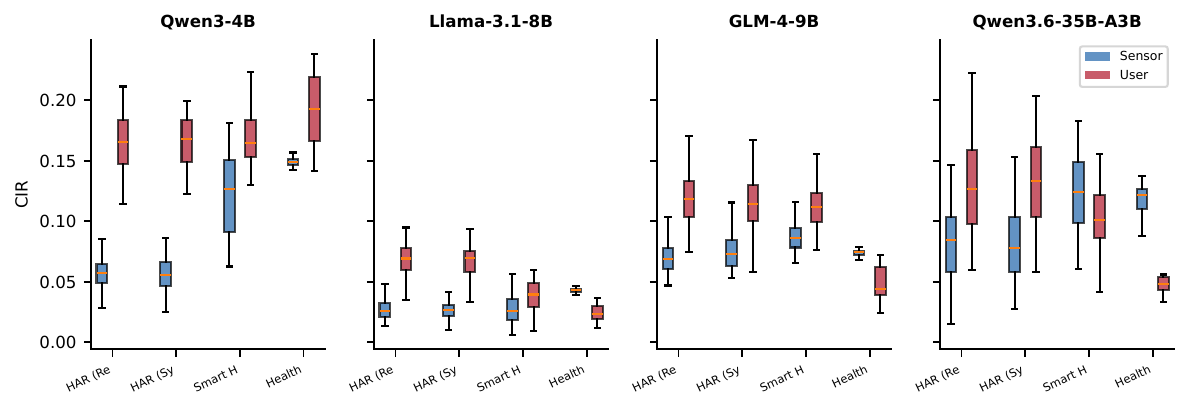}
  \caption{CIR comparison: sensor (blue) vs.\ user (red). On HAR tasks, CIR$_s$ approaches the random baseline, while CIR$_u$ is significantly higher.}
  \label{fig:cir}
  
\end{figure*}

On HAR tasks, all four models' geometric AAI is significantly negative ($-0.664$ to $-0.926$, all $p < 10^{-17}$; see \Cref{fig:aai} for the full distribution). Aggregated across all four models and both HAR datasets, the mean AAI is $-0.805$ ($t = -51.43$, $p = 2.78 \times 10^{-217}$, $N = 576$, Cohen's $d = -2.14$), an extremely large effect size.

CIR comparisons reveal the geometric root cause (\Cref{fig:cir}). On HAR tasks, CIR$_s$ ranges from 0.027 to 0.080, while CIR$_u$ ranges from 0.066 to 0.165 (all paired Wilcoxon $p < 10^{-11}$).

\subsection{Format Interpretability Gradient (H4)}

The behavioral format gradient (from AAI $= -0.97$ extreme inversion to AAI $= +0.25$ equilibrium recovery, summarized in \Cref{tab:format_gradient}) finds a precise correspondence at the geometric representation level: as the input format transitions from purely numerical to semantically rich, context perturbations smoothly shift from ``projecting primarily into the null-space subspace'' to ``effectively injecting into the predictive subspace.''

\begin{table}[h!]
\caption{Format interpretability and authority inversion severity (cross-model average behavioral AAI).}
\label{tab:format_gradient}
\centering\small
\begin{tabular}{llcc}
\toprule
Format Type & Dataset & Mean bAAI & Severity \\
\midrule
IMU numerics & HAR Real & $-$0.969 & Extreme \\
IMU numerics & HAR Synth & $-$0.973 & Extreme \\
Event logs & CASAS & $-$0.651 & Severe \\
Semantic indicators & Health & $+$0.250 & None \\
\bottomrule
\end{tabular}
\end{table}
Because the four datasets differ not only in sensor format but also in task domain, we interpret the cross-dataset gradient as suggestive rather than definitive evidence for a pure format effect. We therefore complement it with within-task format-translation experiments that manipulate representation while holding the underlying sensor signal fixed.

Translation experiments further support this explanation. For Llama-3.1-8B, translating sensor data from numerical to reasoning-based semantic descriptions increases CIR from 0.045 to 0.058 and accuracy from 4.0\% to 15.5\%. Although these improvements do not fully eliminate authority inversion, the direction is consistent.

\subsection{Empirical Validation of Theoretical Predictions}

Theorem~\ref{thm:margin}'s prediction accuracy is satisfactory across all model--task combinations (\Cref{fig:thm2}). The three dense-architecture models achieve Pearson $r > 0.912$ (up to 1.000). The 35B MoE model's $r$ values are slightly lower on HAR tasks (0.786 and 0.840), likely due to the additional nonlinearity introduced by expert routing in MoE architectures. Nonetheless, all $r > 0.78$, indicating practical predictive power across architectures.

\begin{figure*}[htbp]
\centering
\includegraphics[width=\textwidth]{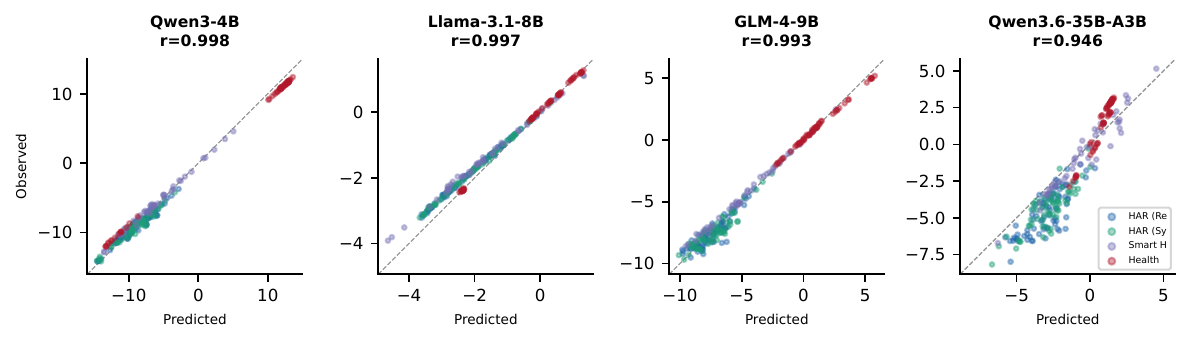}
\caption{Theorem~\ref{thm:margin} validation: predicted vs.\ observed decision margin. Each point is one conflict instance. The diagonal represents perfect prediction.}
\label{fig:thm2}
\end{figure*}

\subsection{Causal Validation: Theory-Guided Repair (H3)}

\begin{table*}[t]
\caption{Causal intervention results per task. Abl.: accuracy after ablation. Inj.: flip rate (wrong $\to$ correct).}
\label{tab:causal}
\centering\small
\begin{tabular}{llccccc}
\toprule
Model & Dataset & Abl.\ $\boldsymbol{\delta}_r^u$ & Abl.\ Rand & Abl.\ $\boldsymbol{\delta}_\bot$ & Inj.\ Theory & Inj.\ Rand \\
\midrule
Qwen3-4B & HAR (Real) & 0.300 & 0.000 & 0.000 & \textbf{0.963} & 0.000 \\
Qwen3-4B & HAR (Synth) & 0.281 & 0.000 & 0.000 & \textbf{0.969} & 0.000 \\
Qwen3-4B & Smart Home & 0.438 & 0.104 & 0.104 & \textbf{0.628} & 0.000 \\
\midrule
Llama-3.1-8B & HAR (Real) & 0.237 & 0.013 & 0.013 & \textbf{0.684} & 0.000 \\
Llama-3.1-8B & HAR (Synth) & 0.250 & 0.016 & 0.016 & \textbf{0.730} & 0.000 \\
Llama-3.1-8B & Smart Home & 0.458 & 0.104 & 0.104 & \textbf{0.725} & 0.000 \\
\midrule
GLM-4-9B & HAR (Real) & 0.125 & 0.000 & 0.000 & \textbf{0.950} & 0.000 \\
GLM-4-9B & HAR (Synth) & 0.047 & 0.000 & 0.000 & \textbf{0.875} & 0.000 \\
GLM-4-9B & Smart Home & 0.604 & 0.083 & 0.083 & \textbf{0.841} & 0.000 \\
\midrule
Qwen3-35B & HAR (Real) & 0.212 & 0.000 & 0.000 & \textbf{0.988} & 0.000 \\
Qwen3-35B & HAR (Synth) & 0.281 & 0.000 & 0.000 & \textbf{1.000} & 0.000 \\
Qwen3-35B & Smart Home & 0.625 & 0.417 & 0.375 & \textbf{1.000} & 0.000 \\
\bottomrule
\end{tabular}
\end{table*}

\begin{figure*}[t]
\centering
\includegraphics[width=\textwidth]{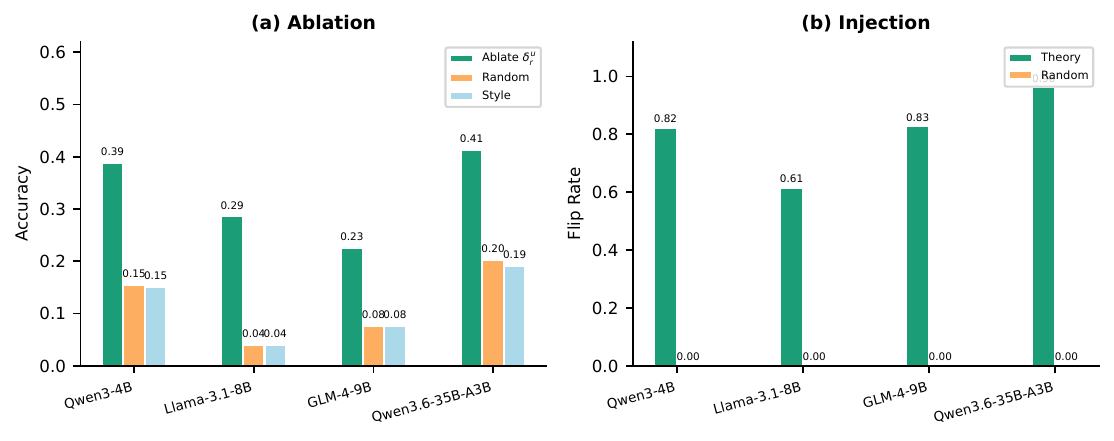}
\caption{Causal interventions. (a)~Ablation: accuracy after removing the user predictive component vs.\ controls. (b)~Injection: flip rate using theory-guided direction vs.\ random direction.}
\label{fig:causal}
\end{figure*}

The causal injection experiments provide the strongest causal support for the theoretical framework in our study (\Cref{tab:causal} and \Cref{fig:causal}). On instances where the model selected the incorrect answer, injecting along the theoretically computed correct decision direction $\mathbf{d}^*/\|\mathbf{d}^*\|$ achieves a cross-model, cross-task average flip rate of \textbf{80.2\%} ($N = 802$). The 35B model's flip rates are particularly striking: 98.8\% on HAR Real, 100\% on both HAR Synth and Smart Home. In contrast, norm-matched random-direction injection achieves a flip rate at the statistical null level ($\leq 0.37\%$ at 95\% confidence, with zero flips observed across all $N = 802$ random injection trials).

This contrast---theory-guided 80.2\% vs.\ random-direction $<$0.4\%---provides strong causal support for the predictive--null-space decomposition as a faithful account of the model's decision mechanism. It demonstrates that the predictive--null-space decomposition is not merely a descriptive geometric tool but can precisely guide targeted modifications to model decisions. When we adjust only a geometrically minimal direction in the residual stream---one computed entirely from the model's own weight matrices without any gradient-based optimization---the model's previously intractable incorrect decision reverses.

Ablation experiments also support the framework, though with more moderate effects. On HAR tasks, ablating the user predictive component $\boldsymbol{\delta}_r^u$ raises accuracy from 0--1.6\% to 4.7--30.0\%, while ablating norm-matched random directions and null-space components leaves accuracy at 0--1.6\%. The moderate post-ablation accuracy is expected: ablating $\boldsymbol{\delta}_r^u$ removes the user claim's push toward the wrong answer but cannot increase the sensor's push toward the correct answer---if the sensor perturbation's predictive component is itself negligible (CIR$_s \approx 0.03$--$0.08$), the model reverts to its parametric prior rather than shifting to the sensor-supported answer.

\subsection{Geometric Authority Calibration (GAC)}

\Cref{tab:gac} compares GAC against alternative correction strategies.

\begin{table}[h]
\caption{Correction strategy comparison (accuracy). GAC $\alpha$: interpolation coefficient (lower = stronger suppression).}
\label{tab:gac}
\centering\small
\begin{tabular}{lccccc}
\toprule
Strategy & Overall & HAR(R) & HAR(S) & CASAS & Health \\
\midrule
\multicolumn{6}{l}{\textit{Qwen3-4B}} \\
No intervention & 0.151 & 0.000 & 0.000 & 0.104 & 0.550 \\
GAC $\alpha{=}0.1$ & \textbf{0.290} & \textbf{0.275} & \textbf{0.234} & 0.125 & 0.500 \\
CoT prompting & 0.194 & 0.037 & 0.125 & 0.042 & 0.600 \\
Semantic trans. & 0.151 & 0.050 & 0.062 & 0.000 & 0.500 \\
\midrule
\multicolumn{6}{l}{\textit{Llama-3.1-8B}} \\
No intervention & 0.040 & 0.013 & 0.016 & 0.104 & 0.050 \\
GAC $\alpha{=}0.1$ & \textbf{0.230} & \textbf{0.225} & \textbf{0.219} & \textbf{0.229} & 0.250 \\
CoT prompting & 0.056 & 0.062 & 0.047 & 0.125 & 0.000 \\
Semantic trans. & 0.155 & 0.150 & 0.141 & 0.062 & 0.250 \\
\bottomrule
\end{tabular}
\end{table}

\begin{figure*}[h]
\centering
\includegraphics[width=\textwidth]{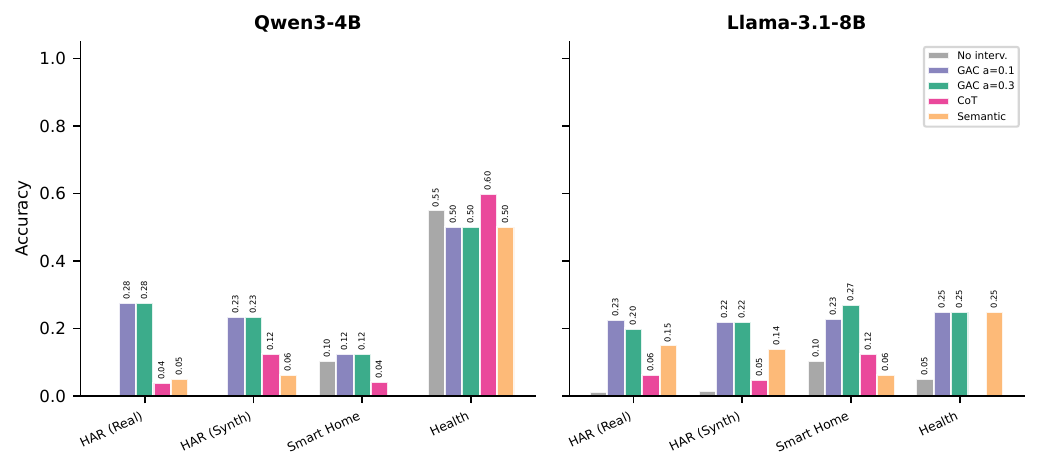}
\caption{Correction strategy comparison across datasets. GAC consistently outperforms CoT prompting and semantic translation on sensor-opaque tasks.}
\label{fig:gac}
\end{figure*}

As shown in \Cref{fig:gac}, GAC significantly outperforms prompt-based strategies on both models. For Qwen3-4B, GAC raises HAR Real accuracy from 0\% to 27.5\%, while CoT achieves only 3.7\% and semantic translation 5.0\%. The failure of CoT to rescue sensor authority aligns with recent findings by Turpin et al.~\cite{turpin2023language}, who demonstrated that LLM CoT explanations are often unfaithful; rather than objectively evaluating conflicting evidence, the model uses CoT to retroactively justify its pre-existing sycophantic preference for the user's claim. For Llama-3.1-8B, GAC's improvement is more uniform: all four datasets reach 21.9--27.1\%.

Safety checks show that under the sensor-only condition (no conflict), GAC reduces accuracy from 28.6--40.5\% to 23.0--29.0\%, indicating that the current implementation partially affects normal context processing alongside user authority suppression. The current latency overhead is 39.5--55.8ms ($\sim$90\%), primarily from the additional baseline forward pass. In deployment, baseline activations can be precomputed and cached, reducing incremental overhead to near zero. We emphasize that GAC in its current form serves as a proof-of-concept demonstrating that representational interventions can mitigate authority inversion; production deployment would require finer-grained, head-level interventions to reduce the safety--efficacy tradeoff documented above.

\subsection{Layer-Level Analysis (H5)}

Layer importance analysis (\Cref{fig:layers}) reveals a consistent pattern across all four models: the layers contributing most to user authority concentrate in the final $\sim$20\% of the network. Qwen3-4B's top-5 layers are [35, 32, 33, 34, 30] (of 36 total); Llama's are [31, 30, 29, 28, 22] (of 32); GLM's are [35, 34, 36, 33, 37] (of 40); and the 35B model's are [39, 38, 36, 35, 37] (of 40).

Cumulative ablation shows that ablating just 1 layer raises accuracy from 33--40\% to 40--60\%. The 35B model shows the largest improvement (40\% $\to$ 60\%), consistent with greater inter-layer redundancy. However, further increasing the number of ablated layers (3, 5) yields no additional improvement, a saturation effect likely reflecting redundant information propagation across layers.

\begin{figure*}[htbp]
\centering
\includegraphics[width=\textwidth]{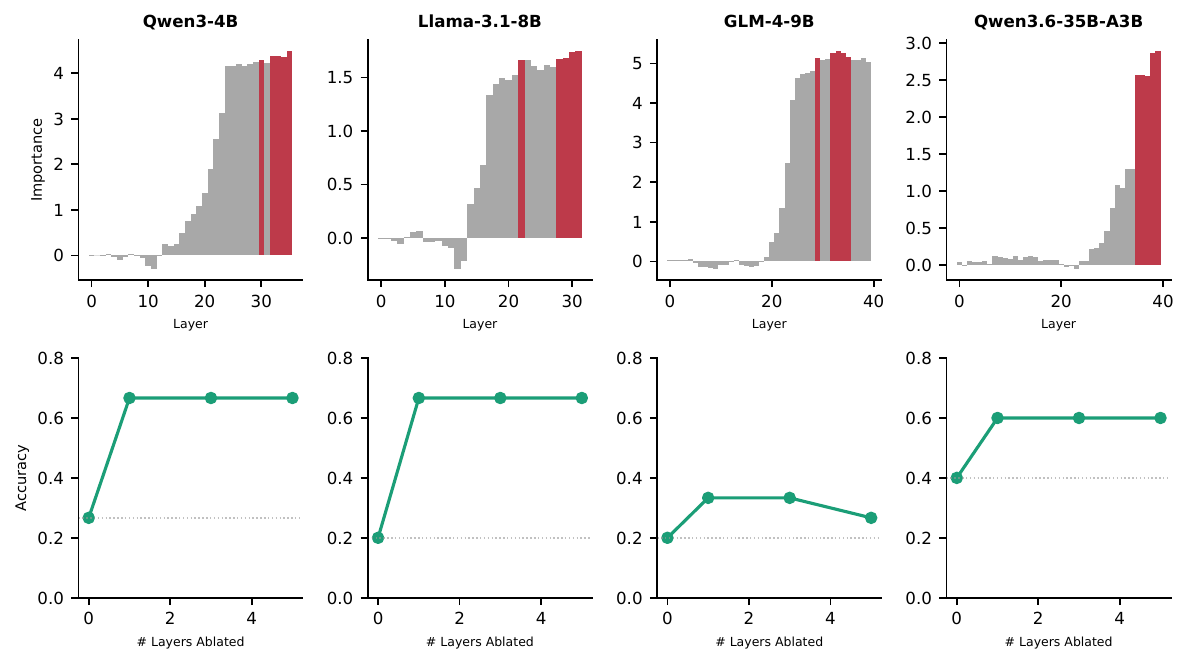}
\caption{Layer-level analysis. Top row: per-layer authority importance (red: top-5 layers). Bottom row: cumulative ablation curves.}
\label{fig:layers}
\end{figure*}

\subsection{Model Scale Effect}

\Cref{fig:scale} summarizes the effect of model scale on authority inversion. On HAR tasks, from 4B to 35B, sensor trust rates remain at 0\%, AAI remains deeply negative ($-0.66$ to $-0.93$), and theory-guided injection flip rates remain above 68\%. This cross-scale consistency indicates that authority inversion is a structural problem in how LLMs process heterogeneous-format information, not a capacity limitation resolvable by scaling.

However, scale is not entirely irrelevant when formats are semi-structured. The 35B model does show improvement on the CASAS task (sensor trust 37.5\% vs.\ 8.3--10.4\% for smaller models). CASAS event logs contain more structured semantic information than raw IMU numerics (e.g., room names in sensor IDs), and larger models may better exploit such semi-structured formats. This further supports the format-dependence hypothesis: the key factor is not merely model size, but whether the specific data format allows larger models to effectively encode it into the predictive subspace.

\begin{figure*}[htbp]
\centering
\includegraphics[width=\textwidth]{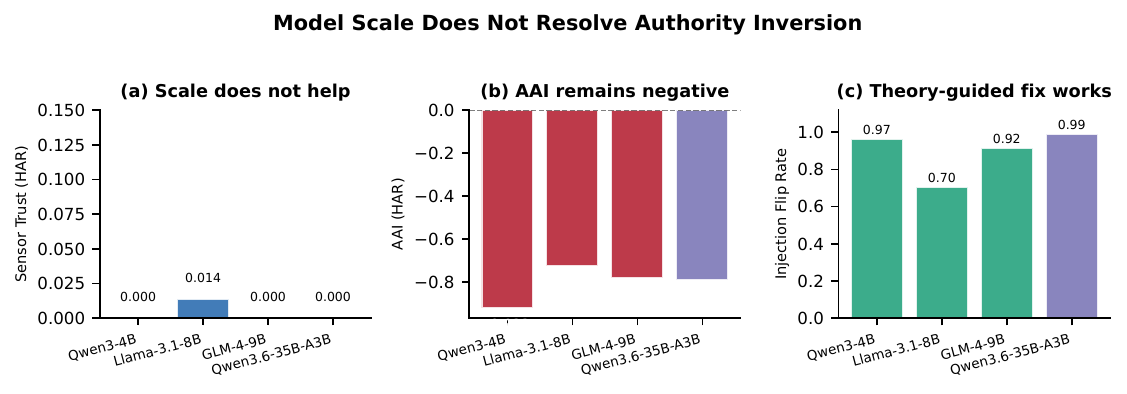}
\caption{Model scale does not resolve authority inversion on numerical-format tasks. (a)~Sensor trust rate remains near zero across scales. (b)~AAI remains deeply negative. (c)~Theory-guided injection remains highly effective.}
\label{fig:scale}
\end{figure*}

\section{Discussion}

\subsection{Summary and Interpretation of Main Findings}

\Cref{fig:summary} consolidates the evidence across all experiments. The core finding of this paper can be stated as a conditional proposition: \textbf{when sensor data is presented in a format that the LLM's token embedding space cannot effectively encode, the model systematically allows the user's natural language claim to dominate the decision, diluting the objective physical measurement's influence below the answer-selection decision boundary.} This is not random noise---on HAR tasks, near-zero sensor trust rates are consistent across four architectures and scales, across two independent HAR datasets.

The geometric analysis provides a mechanism-level explanation. The LLM's output decision is equivalent to directional selection on the hypersphere, and context influence on this direction is transmitted only through the predictive subspace---a subspace of extremely low dimension ($K' \leq 4$). Sensor data's numerical features, after token embedding and Transformer processing, produce perturbations whose directions are nearly orthogonal to the answer subspace (CIR$_s \approx 0.03$--$0.08$), while user natural language claims more effectively inject perturbation energy into the predictive subspace (CIR$_u$ is $1.4$--$2.9\times$ CIR$_s$). This geometric inefficiency in processing numerical formats aligns perfectly with the recently discovered ``topological price'' of discrete logic in LLM representations~\cite{zhang2026geometric}, confirming that numerical sensor data structurally struggles to align with the language-dominated predictive manifold.

\subsection{Disentangling Format Incompatibility from Alignment Bias}
\label{sec:discussion_attribution}

The format gradient across four datasets provides critical evidence for distinguishing two possible mechanisms. The first is \emph{format incompatibility}: the numerical format of sensor data cannot be effectively encoded by the LLM's token embedding space into the predictive subspace. The second is \emph{alignment bias}: RLHF/SFT-aligned models are trained to defer to human users~\cite{perez2023discovering,sharma2023sycophancy}, and this sycophantic inertia may exist independently of sensor format.

Our data suggest that both mechanisms may coexist and mutually amplify, but format incompatibility is the dominant factor. Three lines of evidence support this judgment: (1)~If alignment bias were the sole cause, the health task should exhibit similar inversion severity, yet health AAI is positive ($+0.25$), showing that models can overcome sycophantic tendencies when sensor format is interpretable. (2)~Meta-Llama-3.1-8B is a \emph{base model} that has not undergone instruction tuning or RLHF alignment, yet it exhibits significant authority inversion on HAR tasks (bAAI $= -0.888$), directly demonstrating that format incompatibility alone suffices to cause authority inversion even in the absence of alignment bias. (3)~The geometric observation that CIR$_s$ approaches the random baseline is a purely representation-level measurement independent of alignment training.

Nonetheless, we do not rule out the possibility that alignment training further exacerbates authority inversion. In Instruct models, RLHF's over-optimization for ``human preference'' may preset a baseline margin $m^{(0)}$ biased toward the user. As framed by the signal competition dynamics of social compliance~\cite{zhang2026humanlike}, user claims act as strong sociolinguistic signals that dominate the compliance subspace, ensuring that even if sensor perturbations inject some energy into the predictive subspace, it is insufficient to cross the elevated decision threshold. The interaction between format incompatibility and this compliance bias is a direction worthy of further investigation.

\subsection{Model Scale is Not a Solution}

An important observation of this paper is that, within the range of models tested, scaling parameters from 4B to 35B (MoE architecture) does not mitigate authority inversion on HAR tasks. This result implies that authority inversion is rooted not in insufficient reasoning capacity but in the structural incompatibility between sensor data format and the LLM's token embedding space. Larger models possess stronger reasoning capabilities, but if sensor signals never enter the predictive subspace in the first place, stronger reasoning has nothing to work with.

We note that the 35B model is an MoE architecture whose conditional computation routing mechanism~\cite{jiang2024mixtral} differs fundamentally from the three dense-architecture models. Therefore, this observation cannot be directly generalized as a universal claim that ``scaling laws are ineffective against authority inversion''---whether larger dense-architecture models behave consistently remains to be verified in future work.

The 35B model's improvement on the CASAS task (sensor trust 37.5\% vs.\ 8.3--10.4\%) suggests that the benefit of scale depends on the ``semantic density'' of the sensor format: for completely opaque formats (IMU numerics), scale is ineffective; for semi-structured formats, scale may help.

\subsection{Implications for Ubiquitous Computing System Design}

Our findings carry direct design implications for LLM deployment in ubiquitous computing.

First, \textbf{embedding raw sensor numerics directly into prompts is not a safe practice}. If the sensor format is opaque to the LLM, the model may effectively exclude this information from its decision basis. System designers should use CIR to assess the degree to which a specific sensor format is effectively utilized. CIR$_s$ near the random baseline ($\approx 0.03$) is a clear warning signal.

Second, \textbf{semantic translation is a valuable but insufficient intermediate solution}. Our translation experiments show that converting sensor data to semantically interpretable text can partially increase sensor weight in LLM decisions (e.g., Llama accuracy from 4.0\% to 15.5\%). This aligns with recent 2026 findings by Romero et al.~\cite{ren2026teamllm}, who demonstrated that converting raw multimodal sensor data into structured language contexts helps LLMs overcome the performance bottlenecks of traditional statistical models, a strategy also supported by recent efforts employing LLM-guided semantic alignment for activity recognition~\cite{yan2025lanhar}. However, our results show this improvement varies across models and is not always sufficient. More fundamental solutions require direct intervention in the representation space. Our GAC method provides a proof of concept for inference-time intervention, while the recent development of multimodal architectures like HMotionGPT~\cite{gao2026hmotiongpt}---which directly projects raw IMU signals into the LLM's natural language embedding space at the foundation level---validates that native representation alignment is the promising structural solution to format incompatibility.

Third, \textbf{activity recognition is a particularly high-risk application domain}. Among our four datasets, IMU activity recognition exhibits the most severe and universal authority inversion. This domain extensively uses accelerometer and gyroscope data---precisely the sensor formats most difficult for LLMs to understand.

\subsection{On the Assumption of Sensor Correctness}

A natural objection to our framing is that real-world sensors are not infallible---
they suffer from drift, miscalibration, and spurious activations---and therefore 
an LLM's tendency to trust users over sensors might sometimes constitute 
reasonable fault tolerance rather than a failure mode.

We acknowledge that sensor reliability varies across deployments, and that a 
well-designed system should not blindly trust any single information source. 
However, this objection, while philosophically sound, does not undermine our 
findings for two reasons.

First, our experimental design ensures that sensor data is objectively correct 
in every test instance. The question is not whether the model makes a 
``wrong bet'' on an uncertain signal, but whether the model has the 
\emph{capability} to incorporate sensor evidence when that evidence is 
unambiguously correct. Our results show that it does not: CIR$_s$ values 
approaching the random baseline (0.027--0.031) indicate that numerical sensor 
data fails to enter the decision channel altogether. A system that \emph{cannot} 
use sensor data---even when that data is correct and unambiguous---is 
fundamentally different from a system that \emph{chooses} not to trust 
unreliable sensors after weighing the evidence.

Second, the format-dependence of our findings further refutes the 
``fault tolerance'' interpretation. If the model were performing genuine 
credibility assessment (``sensors might be wrong, so I'll trust the user''), 
this reasoning should apply equally regardless of sensor data format. Instead, 
authority inversion disappears when the same type of physiological information 
is presented in semantically interpretable formats (Health AAI $= +0.25$). 
The model does not distrust sensors \emph{in general}; it specifically fails 
to process \emph{numerically formatted} sensor data. This is an encoding 
limitation, not a deliberate credibility judgment.

We therefore maintain that authority inversion, as defined in this paper, 
represents a genuine capability gap. Whether a deployed system should 
promisingly trust sensors or users in a given scenario is an application-level 
design decision; our contribution is to show that current LLMs lack the 
representational capacity to make this decision in an informed manner when 
sensor data arrives in numerical formats.

\subsection{Human-Centered Implications}

The framing of authority inversion as a failure mode should not be read as an argument that sensors should always override users. In many ubiquitous computing scenarios---pain self-assessment, fatigue reporting, emotional state, subjective comfort---the user's self-report is the most valid source of information, and a system that systematically discounts it would be harmful. The problem we identify is not that LLMs trust users, but that their authority allocation is \emph{implicit, format-dependent, and opaque to the system designer}. Recent human-computer interaction studies have characterized such sycophantic models as ``invisible saboteurs'' that actively mislead users by validating their incorrect beliefs~\cite{invisible_saboteurs_2025}. In ubiquitous sensing, this sabotage becomes literal: an LLM that blindly agrees with a user's subjective feeling over objective physiological distress could fail to trigger critical health alerts.

This distinction has practical consequences. A health monitoring system may reasonably assign high authority to heart-rate sensors for arrhythmia detection but defer to the user's self-report for perceived exertion. Under the current LLM behavior, however, the authority ranking is not determined by application requirements but by the accident of input format: numerical sensor data is systematically underweighted regardless of its clinical importance, while natural-language claims receive disproportionate influence regardless of their reliability.

We therefore argue that the appropriate design response is not to universally boost sensor authority, but to make authority allocation \emph{explicit, application-dependent, and auditable}. The CIR and AAI metrics proposed in this paper provide a first step toward auditability: they allow system designers to measure, before deployment, whether the LLM's implicit authority allocation matches the intended design. GAC provides a proof-of-concept for runtime correction when it does not. More broadly, our findings motivate the development of conflict-handling mechanisms in LLM-mediated ubiquitous systems that expose their authority assumptions to inspection and configuration, rather than burying them in learned representations.

\subsection{Generality of the Theoretical Contribution}

The predictive--null-space decomposition theory and CIR/AAI metrics are not limited to ubiquitous computing. The framework's applicability conditions are: (i)~the model uses RMSNorm or equivalent normalization; (ii)~the decision problem can be expressed as selection over a finite candidate set. These conditions are satisfied in most contemporary LLM applications. The framework can thus extend to other domains involving heterogeneous information source conflicts, such as multimodal reasoning and evidence evaluation in retrieval-augmented generation.

Theorem~\ref{thm:margin}'s joint decision margin formula provides a tool for per-instance decision diagnosis. When a system needs to assess at runtime whether a specific decision may be affected by authority inversion, it can compute $\mathcal{F}_s$, $\mathcal{F}_u$, and AAI without retraining the model.

\begin{figure*}[htbp]
\centering
\includegraphics[width=\textwidth]{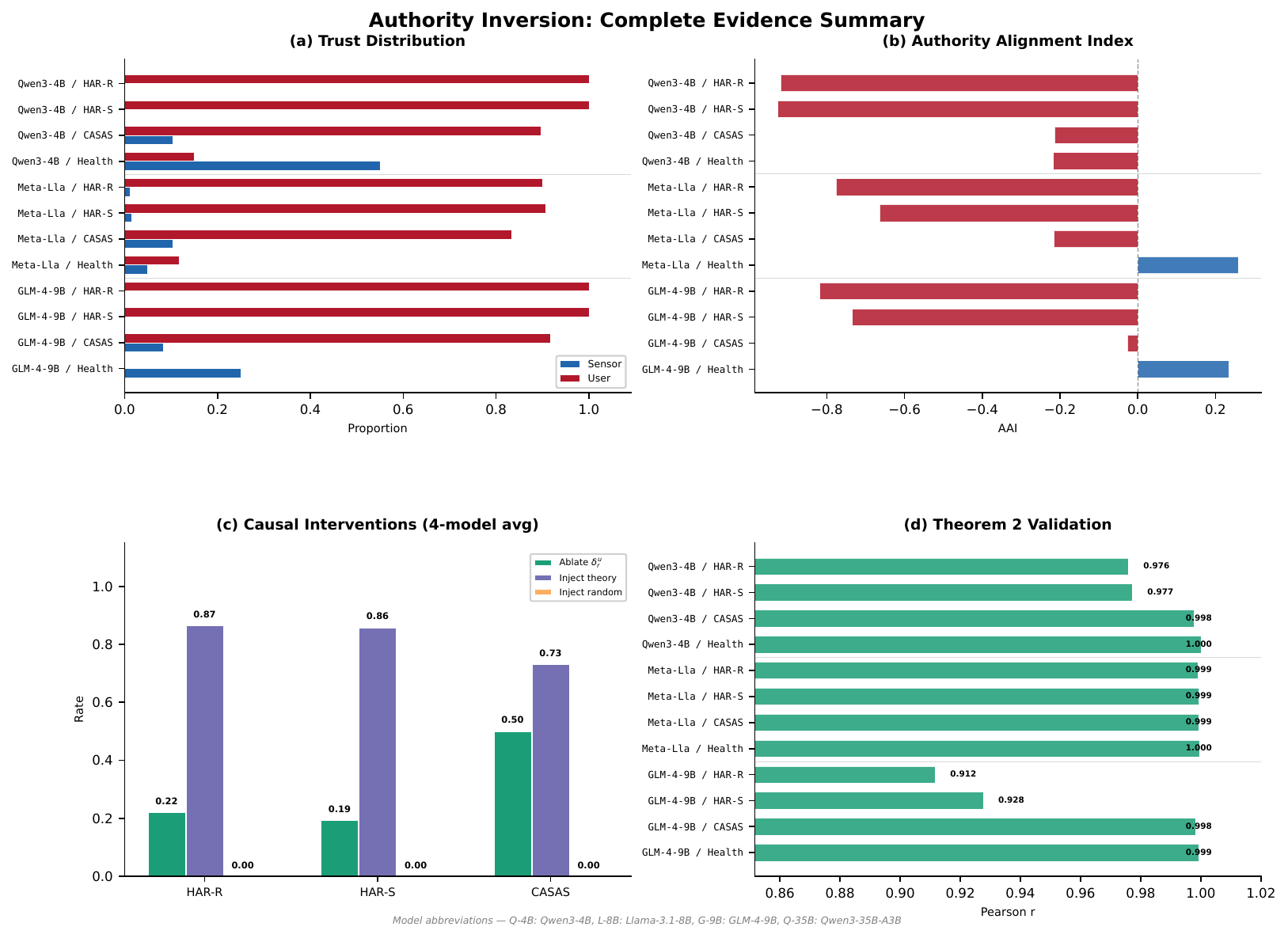}
\caption{Complete evidence summary. (a)~Trust distribution across all model--dataset combinations. (b)~Authority Alignment Index. (c)~Causal intervention results (cross-model average). (d)~Theorem~\ref{thm:margin} validation (Pearson $r$).}
\label{fig:summary}
\end{figure*}

\subsection{Threats to Validity}

\textbf{Construct validity.} Our conflict instances are constructed from templates rather than naturally occurring disagreements. User claims follow fixed linguistic patterns, and sensor data is generated from statistical distributions or clinical reference ranges rather than captured from real deployments. While this ensures unambiguous ground truth and experimental controllability, it may not capture the full diversity of real-world conflicts. Importantly, our core mechanistic finding---that numerical sensor perturbations project into the representational null space (CIR$_s$ approaching the random baseline of 0.031)---is a property of how the model's tokenizer and Transformer layers encode numerical formats, not of the specific data values. This encoding limitation would persist regardless of whether the numerical features originate from generated or real-world sensor streams.

\textbf{Internal validity.} The forced four-choice format (A/B/C/D) enables precise geometric analysis via the answer subspace but differs from open-ended reasoning in deployed systems. Extending CIR and AAI to free-form generation would require a different analytical framework. Additionally, three of our four models are Instruct/Chat variants; although the base model (Llama-3.1-8B) also exhibits authority inversion (see \Cref{sec:discussion_attribution}), we cannot fully rule out alignment training's amplifying effect on Instruct models. An ideal design would pair-compare base and instruct versions of the same model.

\textbf{External validity.} Our benchmark covers three sensor modalities and four datasets, but does not include long-term interaction, continuous decision-making, or closed-loop control scenarios common in ubiquitous deployments. The specific behavioral failure rates (e.g., 0\% sensor trust on HAR) may differ in real deployments where sensor data is presented differently or where the LLM operates as one component in a larger pipeline.

\textbf{Generalizability across architectures.} Theorem~\ref{thm:margin}'s prediction accuracy on the 35B MoE model ($r = 0.786$--$0.840$ on HAR) is lower than on dense models ($r > 0.91$), likely because MoE's conditional expert routing introduces nonlinearity beyond the first-order approximation. Whether larger dense-architecture or closed-source models (e.g., GPT-4, Claude) exhibit similar authority inversion cannot be analyzed with our geometric framework due to inaccessible internal representations; behavioral-level testing can serve as a starting point.

\textbf{GAC deployment constraints.} GAC requires white-box access to model internals (residual stream activations at each layer), per-model identification of critical layers, and an additional baseline forward pass that currently adds 39.5--55.8ms ($\sim$90\%) latency overhead. It is not applicable to closed-source API models. The current layer-level interpolation also reduces non-conflict accuracy by 5--12 percentage points; finer-grained head-level intervention may improve this tradeoff but requires greater computational cost.

\textbf{Prompt and option sensitivity.} User claims are expressed through a limited set of templates, and answer options follow a fixed A/B/C/D ordering. Although the strong effect sizes suggest robustness, richer linguistic variability and randomized option ordering may alter the magnitude of authority inversion. Future work should systematically evaluate prompt sensitivity.

\section{Conclusion}

This paper reveals authority inversion in LLM-powered ubiquitous computing systems: when sensor data is presented in numerical formats, models systematically defer to users' natural language claims, allowing objective sensor evidence's influence in the decision channel to be diluted below the answer-selection boundary. This phenomenon spans four model architectures (4B--35B), three sensor modalities, and four datasets, and does not diminish with increased model scale.

Through a theoretical framework based on hyperspherical directional geometry, we trace the root cause to the inefficient encoding of sensor data perturbations in the answer predictive subspace---sensor information enters the model's representation space but its perturbation energy projects primarily into the high-dimensional null-space subspace (mean CIR$_s$ of only 0.03--0.08). The format interpretability gradient across four datasets further confirms the format dependence of this phenomenon: from completely opaque IMU numerics to semantically rich health indicators, authority inversion severity monotonically decreases.

This finding raises a question of broad significance for the ubiquitous computing community: as LLMs increasingly take on reasoning roles that fuse sensor data with user interaction, what theoretical tools and system architectures are needed to ensure trustworthy behavior under multi-source information conflict? The CIR and AAI metrics proposed in this paper provide pre-deployment diagnostic capability, the predictive--null-space decomposition theory provides mechanism-level understanding, and the GAC method provides inference-time mitigation. We look forward to future work extending these tools in several directions: validating the universality of authority inversion on larger-scale models and more diverse sensor modalities, testing the theoretical framework's applicability on real end-to-end longitudinal deployment data, and developing finer-grained attention-head-level intervention methods to improve GAC's safety--efficacy tradeoff. Additionally, exploring multi-agent LLM consensus mechanisms presents a promising avenue. Drawing inspiration from complex distributed optimization and multi-agent systems~\cite{chen2025confluence, chen2025multiagent}, future architectures could deploy multiple specialized LLMs that negotiate conflicting contextual signals to reach a robust consensus. More broadly, trustworthy conflict handling in LLM-mediated ubiquitous systems cannot be assumed from prompt inclusion or model scale alone; it must be audited and, where necessary, explicitly designed.

\bibliographystyle{ACM-Reference-Format}
\bibliography{references}

\appendix
\section{Proof Details}
\label{app:proofs}

\subsection{Proof of Lemma~\ref{lem:first_order}}

We have $\|\mathbf{h}_c\|^2 = r_0^2 + 2 r_0 \hat{\mathbf{u}}_0^\top\boldsymbol{\delta}_c + \|\boldsymbol{\delta}_c\|^2$. Let $\alpha = \hat{\mathbf{u}}_0^\top\boldsymbol{\delta}_c / r_0$. Then $\|\mathbf{h}_c\| = r_0(1 + \alpha) + O(\epsilon^2 r_0)$, so:
\[
\hat{\mathbf{u}}_c = \frac{\hat{\mathbf{u}}_0 + \boldsymbol{\delta}_c/r_0}{1+\alpha} + O(\epsilon^2) = \hat{\mathbf{u}}_0 + \frac{\boldsymbol{\delta}_c}{r_0} - \frac{\hat{\mathbf{u}}_0^\top\boldsymbol{\delta}_c}{r_0}\hat{\mathbf{u}}_0 + O(\epsilon^2) = \hat{\mathbf{u}}_0 + \frac{1}{r_0}\mathbf{P}_\perp\boldsymbol{\delta}_c + O(\epsilon^2)
\]

\subsection{Proof of Theorem~\ref{thm:decomp}}

Properties (i)--(iii) follow directly from the symmetric idempotency of $\boldsymbol{\Pi}_{\mathcal{A}}$:
\begin{itemize}[nosep]
\item (i) $\boldsymbol{\delta}_r + \boldsymbol{\delta}_\bot = \boldsymbol{\Pi}_{\mathcal{A}}\boldsymbol{\delta}_c^\perp + (\mathbf{I} - \boldsymbol{\Pi}_{\mathcal{A}})\boldsymbol{\delta}_c^\perp = \boldsymbol{\delta}_c^\perp$.
\item (ii) $\boldsymbol{\delta}_r^\top\boldsymbol{\delta}_\bot = (\boldsymbol{\Pi}_{\mathcal{A}}\boldsymbol{\delta}_c^\perp)^\top(\mathbf{I} - \boldsymbol{\Pi}_{\mathcal{A}})\boldsymbol{\delta}_c^\perp = \boldsymbol{\delta}_c^{\perp\top}\boldsymbol{\Pi}_{\mathcal{A}}(\mathbf{I} - \boldsymbol{\Pi}_{\mathcal{A}})\boldsymbol{\delta}_c^\perp = 0$.
\item (iii) Follows from (i) and (ii) by expanding $\|\boldsymbol{\delta}_c^\perp\|^2 = \|\boldsymbol{\delta}_r + \boldsymbol{\delta}_\bot\|^2$.
\end{itemize}
For property (iv): since $\tilde{\mathbf{w}}_{a_k} \in \mathcal{A}$ and $\boldsymbol{\delta}_\bot \in \mathcal{A}^\perp$, we have $\tilde{\mathbf{w}}_{a_k}^\top \boldsymbol{\delta}_\bot = 0$, so $\tilde{\mathbf{w}}_{a_k}^\top \boldsymbol{\delta}_c^\perp = \tilde{\mathbf{w}}_{a_k}^\top \boldsymbol{\delta}_r$.

\end{document}